\documentclass{article} 
\usepackage{iclr2027_conference,times}

\usepackage{amsmath,amsfonts,bm}

\def\eqref#1{equation~\ref{#1}}

\def\1{\bm{1}}

\DeclareMathAlphabet{\mathsfit}{\encodingdefault}{\sfdefault}{m}{sl}
\SetMathAlphabet{\mathsfit}{bold}{\encodingdefault}{\sfdefault}{bx}{n}

\usepackage{hyperref}
\usepackage{url}
\usepackage{amsmath,amssymb,amsthm,bm}
\usepackage{mathtools}
\usepackage{algorithm}
\usepackage{algpseudocode}
\usepackage{tikz}
\usetikzlibrary{arrows.meta}
\usepackage{graphicx}
\usepackage{geometry}
\usepackage{booktabs}
\usepackage{multirow}
\usepackage{xcolor}

\usepackage[table]{xcolor}
\usepackage{booktabs}
\usepackage{graphicx}

\newtheorem{lemma}{Lemma}
\title{GFD-OPD: Guidance-Folded On-Policy Distillation of Diffusion Models Across Scales}

\author{
\centerline{
\textbf{
Zhenxing Zhang\textsuperscript{1,\,*,\,\textdagger},
Jiayan Teng\textsuperscript{2,\,*,\,\textdaggerdbl},
Wenxu Wu\textsuperscript{3},
Zhuoyi Yang\textsuperscript{2},
Jiazheng Xu\textsuperscript{2},
}
}\\
\centerline{
\textbf{
Wendi Zheng\textsuperscript{2},
Jie Tang\textsuperscript{2},
Dan Guo\textsuperscript{1},
Meng Wang\textsuperscript{1,\,\textsection}
}
}\\
\centerline{
\textsuperscript{1}Hefei University of Technology \quad
\textsuperscript{2}Tsinghua University \quad
\textsuperscript{3}Zhipu AI
}
}
\newcommand{\cfg}{\mathrm{cfg}}

\iclrfinalcopy 
\begin{document}
\pagestyle{plain} 

\maketitle
{\renewcommand{\thefootnote}{\fnsymbol{footnote}}
\footnotetext[1]{Equal contribution.}
\footnotetext[2]{Work done during internship at Zhipu AI.}
\footnotetext[3]{Technical lead.}
\footnotetext[4]{Corresponding author: \texttt{eric.mengwang@gmail.com}}
}

\begin{abstract}
On-policy distillation (OPD) has demonstrated two important capabilities in language models: compressing large teachers into smaller students and merging expert models into a single model. Existing diffusion OPD, however, mostly focus on the latter, with teachers and students sharing the same backbone and scale. We investigate large-to-small diffusion opd from large teachers to a small student and find that the standard recipe fails. To find the underlying cause, we propose \textbf{Fixed-State KL}, an effective and fair way to measure the distribution gap between student and teacher during OPD training for diffusion models. We are the first to clarify why large-to-small OPD is challenging for diffusion models: a smaller student struggles to perfectly match the distribution of a larger teacher, while classifier-free guidance can accumulate and amplify the distributional discrepancies between the student’s conditional and unconditional branches and those of the teacher. To solve this problem, we propose \textbf{GFD-OPD}, a simple yet effective method that reduces the student–teacher gap while avoiding the error amplification of the CFG composition. Across numerous experiments, GFD outperforms previous baselines in both training efficiency and final performance, achieving state-of-the-art results on all benchmarks. The code and checkpoints for this study are available at \href{https://github.com/CriliasMiller/GFD-OPD}{\textit{\textcolor{blue}{here}}}. 
\end{abstract}

\begin{figure}[htb]
\vspace{-1em}
\centering
\includegraphics[width=\linewidth]{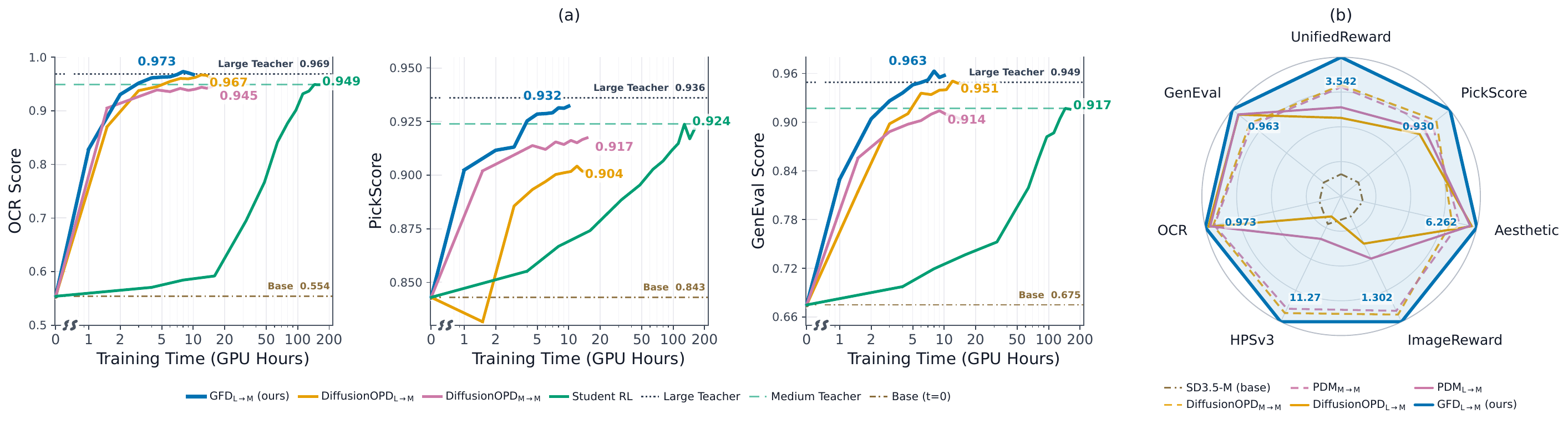}
\vspace{-1em}
\caption{GFD-OPD distills SD3.5-Large into SD3.5-Medium.
(a)~Score versus training compute (GPU hours) on OCR, PickScore, and
GenEval: GFD-OPD overtakes all baselines within $10$ GPU-hours.
(b)~GFD-OPD ranks first on all seven evaluation dimensions, covering both
task rewards and image quality.}
\label{fig:eff}
\end{figure}
\section{Introduction}\label{sec:intro}

On-policy distillation (OPD) has recently emerged as a powerful post-training recipe for large
language models. Like reinforcement learning (RL), OPD trains the student on trajectories sampled
from its own policy and therefore optimizes the student's own distribution; unlike RL, which has to
work with sparse and delayed rewards, OPD lets a teacher provide dense supervision at every step of
the student-generated trajectory. This fine-grained feedback largely removes the credit-assignment
difficulty of RL, so OPD enjoys a more stable optimization and a much better sample efficiency at a
comparable final quality.

In language models OPD is primarily applied in two paradigms. The first is \emph{strong-to-weak}
distillation, where the capability of a large teacher is transferred to a small student that is
cheap to deploy~\citep{opd}. The second is \emph{capability merging}, where several specialist
models derived from the same base model are merged into a single student without conflicts~\citep{glm5}.
Diffusion models have begun to adopt OPD as a post-training method as
well~\citep{diffusionopd,flowopd}, but the existing diffusion work mostly focus on the second paradigm. 
The first paradigm, distilling a large diffusion model into a small one, remains unexplored. 
It is also important: a compact model that inherits capabilities of a larger teacher can achieve near-teacher quality at inference while retaining a low cost close to that of the smaller model.

Our goal is to investigate effective approaches for distilling the capabilities of large diffusion models into smaller ones. Although DiffusionOPD performs effectively when merging models with the same architecture, we find directly extending it to large-to-small distillation leads to severe degradation. As shown in Fig.~\ref{fig:artifacts}, generated images exhibit artifacts, such as corrupted regions, loss of fine-grained textures and so on.

To understand the source of this degradation, we propose \textbf{Fixed-State KL} for measuring the distributional discrepancy between the student and teacher.  Specifically, we compute KL divergence on states sampled from normal teacher trajectories and noise-corrupted trajectories constructed from real images, providing a fair and well-behaved evaluation distribution for all methods.

Under this evaluation metric, we find that large-to-small DiffusionOPD produces a substantially larger distributional discrepancy between the student and teacher than its same-architecture counterpart. Furthermore, we quantitatively characterize models' final sampling result through spectral energy analysis of the final latent representations and high-pass texture analysis of the decoded images, which quantify the deviation from the teacher at both coarse-grained (e.g., global tone) and fine-grained (e.g., local texture) levels.

Through extensive analysis, we indicate that a smaller-capacity student consistently struggles to perfectly match the distribution of a larger teacher. Besides, classifier-free guidance (CFG), used during diffusion sampling, can accumulate and amplify the distributional discrepancies between the student's conditional and unconditional branches and those of the teacher, causing a larger deviation in the final generation trajectory.

Motivated by these observations, we propose a simple yet effective method, \textbf{Guidance Folding Distillation (GFD)}. GFD reduces the distributional discrepancy between the student and teacher while avoiding the error amplification introduced by CFG composition, enabling effective large-to-small OPD for diffusion models.

Our contributions are as follows.
\begin{itemize}
\item We propose \textbf{Fixed-State KL}, an effective and fair way to measure the distribution gap between student and teacher during OPD training for diffusion models.
\item Through multifaceted analysis, we are the first to clarify \textbf{why large-to-small OPD is challenging for diffusion models}. To address these challenges, we propose \textbf{GFD}, a simple and effective method that reduces the student--teacher gap while avoiding the error amplification of the CFG composition.
\item Through multiple experiments, combined with quantitative and qualitative analysis, we show the effectiveness of GFD. GFD removes the artifacts in cross-scale DiffusionOPD and outperforms previous baselines in both training efficiency and final performance.
\end{itemize}\
\vspace{-1.5em}

\section{method}\label{sec:motivation}

\subsection{Preliminary: On-Policy Distillation in Language Models and Diffusion Models}\label{sec:prelim}
On-policy distillation (OPD) trains the student on trajectories sampled from its own policy, while the teacher provides dense supervision on every visited state. In language models, OPD minimizes the reverse KL between the student policy $\pi_{\bm\theta}$ and the frozen teacher $\pi^\star$ along student-generated sequences:
\begin{equation}
\mathcal{L}_{\mathrm{OPD}}^{\mathrm{LLM}}(\bm\theta)
= \mathbb{E}_{\bm y\sim\pi_{\bm\theta}}
\Big[\sum_{t=1}^{|\bm y|}
\mathrm{KL}\big(\pi_{\bm\theta}(\cdot\mid \bm y_{<t})\,\|\,\pi^\star(\cdot\mid \bm y_{<t})\big)\Big].
\label{eq:opd-llm}
\end{equation}

\vspace{-0.8em}
DiffusionOPD transfers the same idea to diffusion and flow-matching samplers. At each denoising step $j$, the student and teacher induce Gaussian transition kernels $p_S(x_{t_{j+1}}\mid x_{t_j})$ and $p_T(x_{t_{j+1}}\mid x_{t_j})$ whose covariance $\sigma_j^2\bm I$ is fixed by the shared sampler, so the per-step reverse KL has the closed form $\|\mu_S(x_{t_j};\bm\theta)-\mu_T(x_{t_j})\|_2^2/(2\sigma_j^2)$. Since the transition mean is affine in the velocity field with sampler-fixed coefficients, and sampling deploys the CFG-guided velocity, DiffusionOPD matches the CFG-guided output to the teacher's:
\begin{equation}
\mathcal{L}_{\mathrm{OPD}}^{\mathrm{diff}}(\bm\theta)
=\mathbb{E}_{x_{0:N}\sim p_{S,\bm\theta}}
\left[\sum_{j=0}^{N-1} \omega_j
\big\|v^g_S(x_{t_j},t_j,\bm c)-v^g_T(x_{t_j},t_j,\bm c)\big\|_2^2\right],
\label{eq:opd-diff}
\end{equation}
where $v^{g}_{\star}=v^{u}_{\star}+w\,(v^{c}_{\star}-v^{u}_{\star})$, with $\star\!\in\!\{S,T\}$ indexing the student and the teacher, is the guided velocity composed from the conditional and unconditional branches $v^{c}_{\star},v^{u}_{\star}$ at guidance scale $w$, and $\omega_j=\eta_j^{2}/(2\sigma_j^{2})$ is fixed by the sampler, with $\eta_j$ the coefficient of the velocity in the transition mean and $\sigma_j^{2}$ the per-step variance; $\eta_j=\big(1+\sigma_{t_j}^{2}(1-t_j)/(2t_j)\big)\Delta t_j$ and $\sigma_j^{2}=\sigma_{t_j}^{2}\Delta t_j$. In both SDE and ODE settings, DiffusionOPD is thus a pathwise, closed-form velocity-matching problem along the student's rollout.

\subsection{Fixed-state KL Evaluation}\label{sec:kl_eval}
We first directly apply the standard DiffusionOPD objective to the large-to-small OPD setting. However, after training, from Fig.~\ref{fig:artifacts} we can see that the resulting images exhibit obvious visual artifacts: they contain corrupted textures and are generally darker. To further quantify the degradation in generation quality, we conduct a spectral energy analysis on the final latent representations and a high-pass texture analysis on the decoded images. As shown in Fig.~\ref{fig:w_sweep}, DiffusionOPD$_{\mathrm{L\rightarrow M}}$ deviates substantially from the teacher across all evaluated aspects.

To investigate the underlying cause, we need to measure the distributional discrepancy between the student and teacher. A natural choice is to use the objective in OPD training: the reverse KL divergence between the student and teacher distributions, computed along the trajectory rolled out by the student. However, although this formulation is well suited as a training objective, we argue that it is not an appropriate metric to evaluate the distributional discrepancy between the student and the teacher after training.

There are two main reasons. First, different training methods generally produce different student policies. Computing the KL divergence on each student's own trajectories therefore evaluates the models on different sets of states, confounding the effect of distributional mismatch with differences in state visitation. It makes the resulting metrics difficult to compare in a fair and controlled manner. 

Second, the student model may deviate from the normal generation trajectory during training and enter abnormal states. As pointed out by~\cite{xin2026escaping}, the teacher is no longer guaranteed to provide reliable predictions in such states,  since the queried states may lie far outside the region on which the teacher's behavior is meaningful. Consequently, the KL divergence computed on these abnormal states may no longer faithfully reflect the discrepancy between a well-behaved student and the teacher.

\begin{figure}[t]
\vspace{-0.5em}
\begin{minipage}[t]{0.36\linewidth}
\vspace{0pt}\centering
\includegraphics[width=\linewidth]{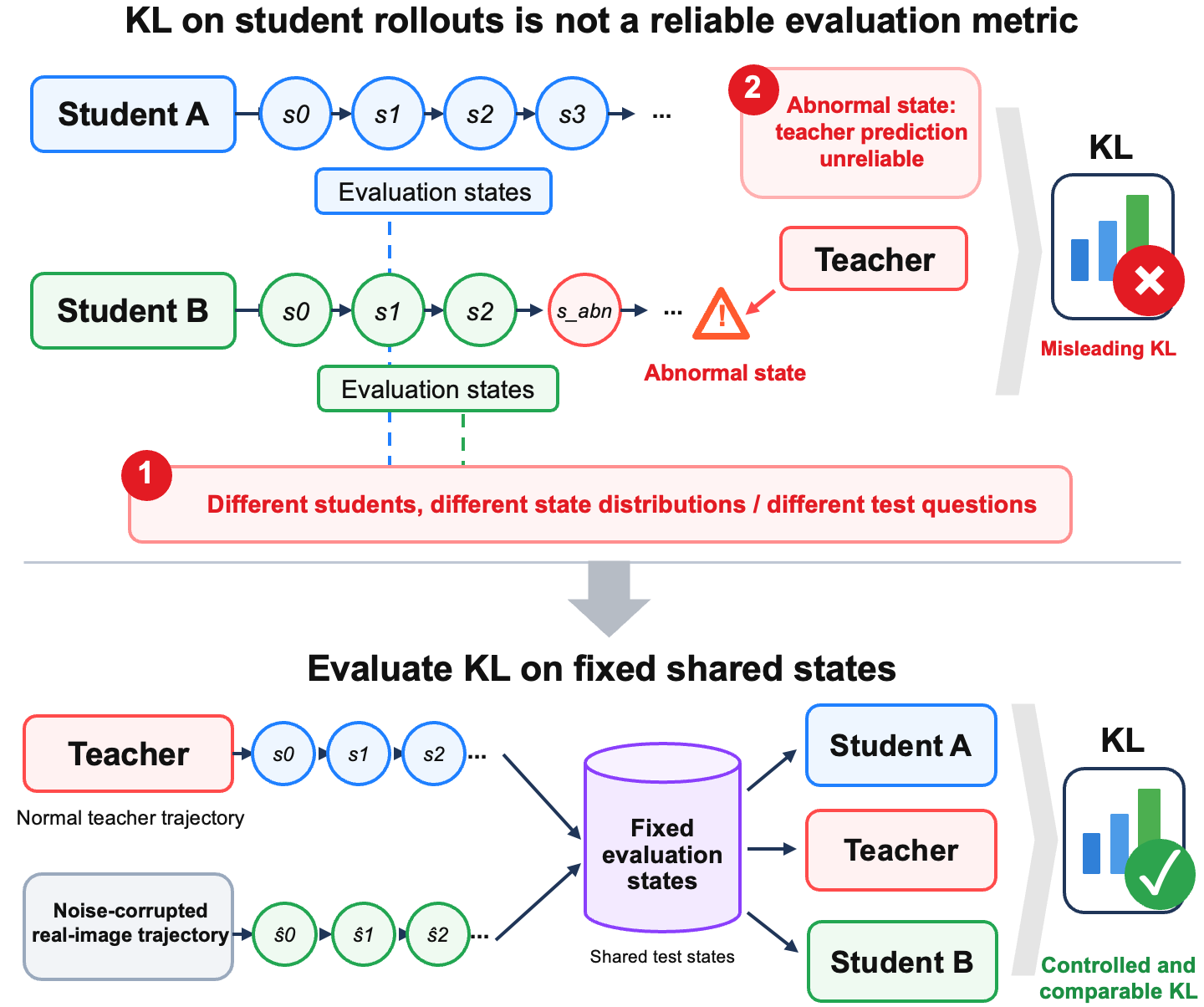}
\vspace{-2em}
\caption{\small Illustration of the KL evaluation.}
\label{fig:kl_eval}
\end{minipage}\hfill
\begin{minipage}[t]{0.57\linewidth}
\vspace{0pt}\vspace{-\abovecaptionskip}\centering\small
\makeatletter\def\@captype{table}\makeatother
\caption{KL divergence between each student and its corresponding teacher. $\mathrm{P_{stu}}$ and $\mathrm{P_{tea}}$ denote sampling trajectories generated by the student and teacher models, respectively, while $\mathrm{P_{real}}$ denotes trajectories constructed by adding noise to real images.}
\label{tab:kl-eval}\label{tab:kl_result}
\resizebox{\linewidth}{!}{%
\begin{tabular}{l c c c cc}
\toprule
& KL ($\mathrm{P_{tea}}$) & KL ($\mathrm{P_{real}}$) & KL ($\mathrm{P_{stu}}$) & \multicolumn{2}{c}{branch KL ($\mathrm{P_{tea}}$)}\\
\cmidrule(lr){2-2}\cmidrule(lr){3-3}\cmidrule(lr){4-4}\cmidrule(lr){5-6}
Student & guided & guided & guided & cond & uncond\\
\midrule
DiffusionOPD$_{\mathrm{M\rightarrow M}}$  & 0.041 & 0.022 & 0.039 & 0.017 & 0.023\\
\midrule\midrule
DiffusionOPD$_{\mathrm{L\rightarrow M}}$ & 0.068 & 0.046 & 0.057 & 0.037 & 0.039\\
PDM$_{\mathrm{L\rightarrow M}}$          & 0.063 & 0.046 & 0.052 & 0.023 & 0.022\\
split-KL$_{\mathrm{L\rightarrow M}}$     & 0.074 & 0.049 & 0.065 & \textbf{0.021} & \textbf{0.020}\\
\midrule
GFD$_{\mathrm{L\rightarrow M}}$ ($w\!=\!1$)  & \textbf{0.050} & \textbf{0.030} & \textbf{0.043} & -- & --\\
GFD$_{\mathrm{L\rightarrow M}}$ ($w\!=\!2$)  & 0.053 & 0.032 & 0.053 & -- & --\\
\bottomrule
\end{tabular}}
\end{minipage}
\vspace{-1em}
\end{figure}

Motivated by these considerations, we propose \textbf{Fixed-State KL}, which is performed on a fixed and normal set of states shared by all student models. Specifically, we use states from normal teacher trajectories, together with noise-corrupted trajectories constructed from real images, as common evaluation inputs for computing the KL divergence between the student and teacher. By decoupling the evaluation states from the student's own rollout distribution, this protocol enables a more controlled and comparable assessment of how closely different students match the teacher distribution.

Furthermore, as shown in Tab.~\ref{tab:kl_result}, the KL divergence evaluated on student rollout trajectories exhibits a markedly different trend from that computed on teacher rollout trajectories and noise-corrupted trajectories constructed from real images. The former fails to reflect the true trend of student--teacher distributional discrepancy. It may provide a misleading evaluation signal.

\subsection{From Cross-scale DiffusionOPD to GFD}

Based on the KL evaluation above, we can figure out why the performance of DiffusionOPD in $L\!\rightarrow\!M$ setting degrade severely.
As shown in Table~\ref{tab:kl_result}, the guided-output KL, conditional KL, and unconditional KL of DiffusionOPD$_{\mathrm{L\rightarrow M}}$ are all higher than those of DiffusionOPD$_{\mathrm{M\rightarrow M}}$. ``guided output'' refers to the CFG-combined output used during sampling. Accordingly, the final evaluation should focus on the distributional KL divergence between the student’s guided output and the teacher’s guided output.

\begin{figure}[h]
\centering
\includegraphics[width=\linewidth]{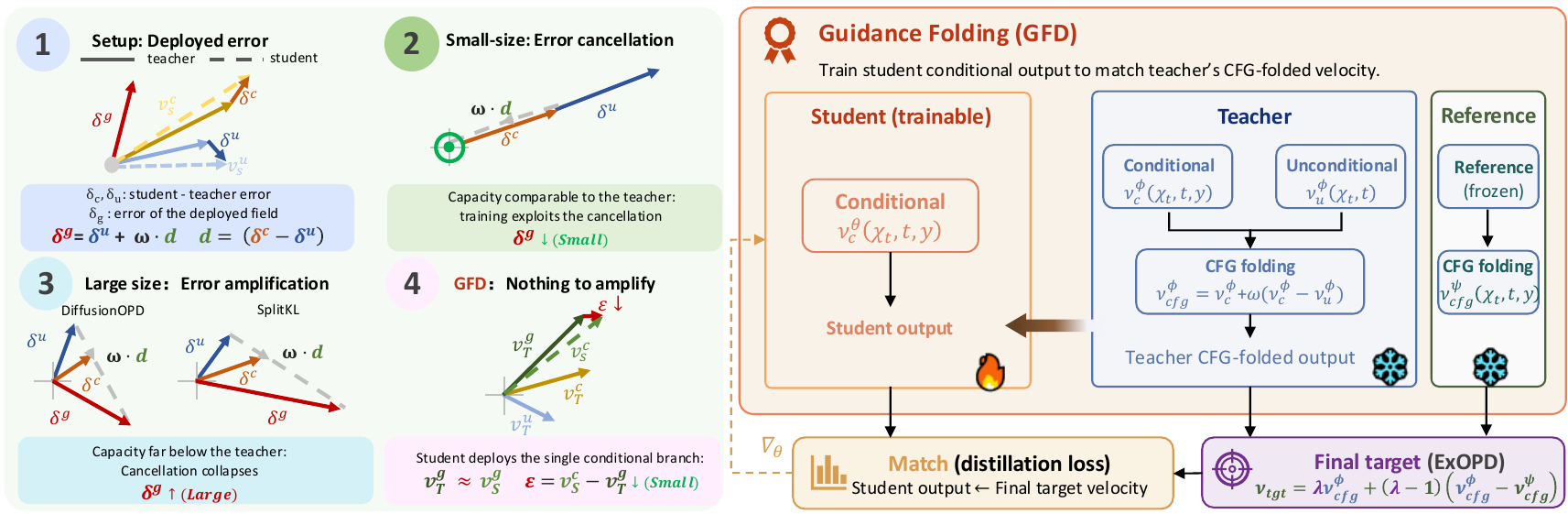}
\vspace{-2em}
\caption{\textbf{GFD-OPD overview.} \textbf{Left} (motivation): the deployed error of a CFG-composed student decomposes as $\delta_g=\delta_u+\omega\,d$ --- same-size distillation survives through error cancellation, large-to-small does not, and guidance folding removes the composition so that nothing is left for $\omega$ to amplify. \textbf{Right} (method): the process of GFD-OPD.}
\label{fig:overview}
\vspace{-0.5em}
\end{figure}

\paragraph{Attempting branch-wise supervision}
Recall that DiffusionOPD objective directly matches the student's guided prediction $v^g_S$ to the teacher's guided prediction $v^g_T$. It lacks separately constraining of the conditional and unconditional branches. Such an objective may be sufficient when the student has comparable model capacity to the teacher, as in the $M\!\rightarrow\!M$ setting. However, in a large-to-small setting, the smaller student has limited approximation capacity, which may make the conditional and unconditional branches more prone to drifting away from their teacher counterparts. Such branch-wise drift may in turn lead to a larger discrepancy in the final guided output.

This observation motivates us to explicitly supervise the conditional and unconditional branches separately, thereby preventing them from drifting. A natural choice is to adopt the following training objective:
\begin{equation}
    \mathcal{L}
    =\left\| v_S^c - v_T^c \right\|_2^2
    +
    \left\| v_S^u - v_T^u \right\|_2^2,
\end{equation}
\vspace{-1.5em}

PDM~\citep{pdmopd} also propose a similar objective as 
\begin{equation}
    \mathcal{L}
    =\left\| v_S^c - v_T^c \right\|_2^2
    +
    \alpha\left\| v_S^g - v_T^g \right\|_2^2,
\end{equation}
However, their motivation is fundamentally different from ours: we use this objective to prevent the conditional and unconditional branches from drifting in the large-to-small setting but theirs are not.

After training with separate supervision on the two branches, from Table~\ref{tab:kl_result}, we observe that the KL divergences of both the conditional and unconditional branches are indeed reduced. However, the KL divergence of the final guided output does not decrease accordingly. Consistent with this observation, the generated images in Fig.~\ref{fig:artifacts} still exhibit noticeable quality degradation. As shown in Fig.~\ref{fig:w_sweep}, the corresponding quantitative metrics also remain substantially deviated from those of the teacher.

\paragraph{Why the branch error induce but the guided field does not?}

\begin{figure}[htbp]
\vspace{-1em}
\centering
\includegraphics[width=0.9 \linewidth]{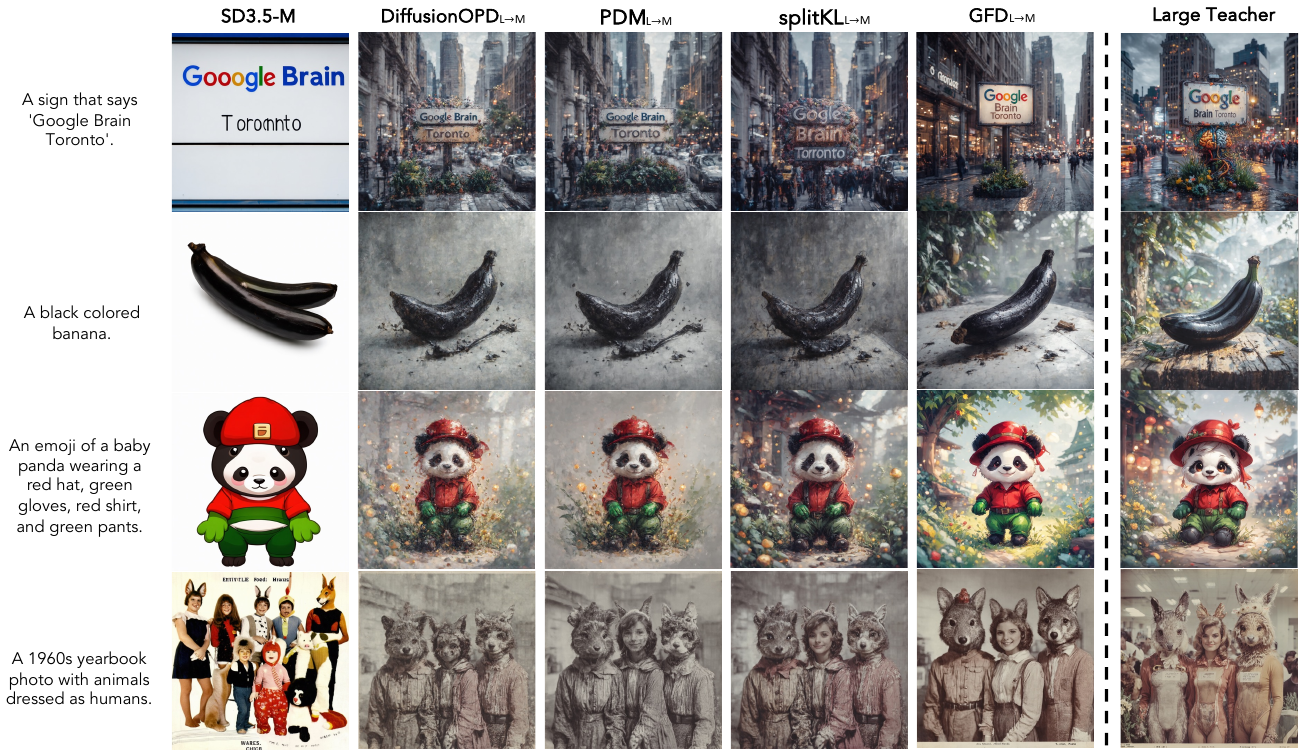}
\vspace{-0.5em}
\caption{Shared cross-architecture artifacts under a SD3.5-L teacher. Rows:
the SD3.5-M base, three prior OPD recipes (\emph{DiffusionOPD},
\emph{PDM}, \emph{splitKL}), GFD and large teacher. }
\label{fig:artifacts}
\end{figure}

To investigate the cause of the increased guided-output KL, we decompose the error of the guided prediction as
\begin{equation}
    \delta^g = \delta^u + w d,
\end{equation}

\vspace{-0.5em}
where
$
\delta^g = v^g_S - v^g_T,
\delta^u = v^u_S - v^u_T,
\delta^c = v^c_S - v^c_T,
$
and
$
d = \delta^c - \delta^u.
$
Table~\ref{tab:error_decompose} reveals that, when trained to directly match the guided output, DiffusionOPD$_{\mathrm{M\rightarrow M}}$ is able to exploit the cancellation between the two error components, $\delta^u$ and $wd$, resulting in a relatively small final guided error $\delta^g$. In contrast, for DiffusionOPD$_{\mathrm{L\rightarrow M}}$, the limited approximation capacity of the smaller student weakens such error cancellation. For PDM and SplitKL, separately supervising the conditional and unconditional branches largely removes this cancellation effect, causing the errors to be amplified rather than canceled after guidance composition. We provide a more detailed analysis in Section~\ref{subsec:errpr_decompose}.

This observation further suggests that, in the large-to-small setting, discrepancies in the conditional and unconditional branches can be systematically amplified by the guidance operation, ultimately leading to a larger shift in the final generative distribution. 

\paragraph{Guidance Folding for Large-to-Small Distillation}
This motivates a different design: directly absorb the teacher's guided policy into the student's conditional branch, eliminating the need to recover the target distribution through the composition of two imperfectly matched branches. We propose \textbf{Guidance Folding Distillation (GFD)}. Rather than jointly optimizing the student's conditional and unconditional distributions, it converts the original two-branch composition problem into a single-branch velocity matching problem.

Formally, GFD minimize the discrepancy between the student's conditional distribution and the teacher's guided distribution:
\begin{equation}
    \min_{\theta}
    \mathbb{E}_{x_t \sim d^S}
    \left[
        D_{\mathrm{KL}}
        \left(
            \pi^c_S(\cdot \mid x_t, c)
            \,\Vert\,
            \pi^g_T(\cdot \mid x_t, c)
        \right)
    \right].
\end{equation}

Applying the per-step derivation of Eq.~\ref{eq:opd-diff}, the GFD objective can be written explicitly as
\begin{equation}
\mathcal{L}_{\mathrm{GFD}}(\bm\theta)
=\mathbb{E}_{x_{0:N}\sim p_{S,\bm\theta}}
\left[\sum_{j=0}^{N-1}
\omega_j
\big\|v_S^{c}(x_{t_j},t_j,\bm c;\bm\theta)-v_T^{g}(x_{t_j},t_j;w)\big\|_2^2\right].
\label{eq:gfd}
\end{equation}
Table~\ref{tab:kl_result} shows that GFD achieves substantially lower KL divergence than the other methods. As illustrated in Fig.~\ref{fig:artifacts} and Fig.~\ref{fig:w_sweep}, GFD also restores normal generation quality, and the generated images are closer to those of the teacher in both coarse-grained and fine-grained metrics.

But a small, systematic gap remains. Fig.~\ref{fig:w_sweep} shows that GFD ($w\!=\!1$) has slightly weaker spectral energy than the teacher's guided output. And Appendix~\ref{sec:app-proj} shows that the student learns the direction of the teacher's field but shrinks its magnitude toward its own prior. At inference this can be patched with a slightly larger guidance scale, with $w\!=\!2$ the student distribution moves closer to the teacher's. GFD does not constrain the unconditional branch during training. In the low-noise regime, the unconditional branch co-evolves with the conditional branch and exhibits similar shifts, resulting in the CFG only containing a small increment and avoiding excessive guidance. Although the trajectory-level KL increases slightly at \(w\!=\!2\), we attribute this to the limited sensitivity of this metric to differences at such a small scale. 

\subsection{Reward Extrapolation for Large-to-Small Distillation}\label{sec:gopd}

\citet{exopd} show that OPD in language models is a KL-regularized RL objective in which the teacher enters through a dense reward, and that scaling this reward by a factor $\lambda>1$ moves the distillation target away from the reference model and beyond the teacher. We adopt this idea to further improve the student: by extrapolating the target in the direction away from the student's prior, the trained field is brought closer to the teacher without relying on an inference-time guidance increment.

In language models, the ExOPD objective is equivalent to standard OPD with the teacher replaced by the target $\pi_{\mathrm{tgt}}\propto(\pi^\star)^{\lambda}(\pi_{\mathrm{ref}})^{1-\lambda}$, where $\pi_{\mathrm{ref}}$ is a frozen reference policy. We apply the same replacement to the per-step transition kernels of diffusion OPD: $p_{\mathrm{tgt}}\propto p_T^{\lambda}\,p_R^{1-\lambda}$, where $p_R$ is the kernel induced by the reference velocity $v_R$. Since both kernels are Gaussians with the same sampler-fixed variance, $p_{\mathrm{tgt}}$ remains a valid Gaussian kernel with that variance for every $\lambda$, and only its mean moves, corresponding to the extrapolated velocity (Appendix~\ref{app:exopd}):

\begin{equation}
v_{\mathrm{tgt}}
=\lambda\, v_T+(1-\lambda)\,v_R
=v_T+(\lambda-1)(v_T-v_R).
\label{eq:v-tgt}
\end{equation}


\vspace{-1em}
\section{Experiments}\label{sec:exp}


\subsection{Experimental Setup}\label{sec:setup}
\begin{table}[t]
\vspace{-2em}
\centering
\begin{minipage}[t]{0.615\textwidth}
\centering
\begin{minipage}[t][10mm][t]{\linewidth}
\caption{SD3.5 merged multi-objective OPD results.}
\label{tab:main}
\end{minipage}\\
\renewcommand{\arraystretch}{1.39}
\resizebox{\linewidth}{!}{%
\begin{tabular}{@{}l ccc cccc@{}}
\toprule
& \multicolumn{3}{c}{Task rewards} & \multicolumn{4}{c}{Image quality} \\
\cmidrule(lr){2-4}\cmidrule(lr){5-8}
Model & GenEval & OCR & \begin{tabular}[c]{@{}c@{}}Pick\\Score\end{tabular} & \begin{tabular}[c]{@{}c@{}}Aes-\\thetic\end{tabular} & \begin{tabular}[c]{@{}c@{}}Image\\Reward\end{tabular} & \begin{tabular}[c]{@{}c@{}}Unified\\Reward\end{tabular} & HPSv3 \\
\midrule
SD3.5-M (base) & 0.675 & 0.554 & 0.843 & 5.373 & 0.907 & 3.201 & 9.020 \\
SD3.5-L (base) & 0.720 & 0.716 & 0.854 & 5.527 & 1.025 & 3.289 & 9.853 \\
\midrule
\multicolumn{8}{@{}l}{\emph{Per-task Large teachers}}\\
\quad GenEval-Teacher & 0.949 & 0.801 & 0.834 & 5.463 & 1.144 & 3.416 & 9.812 \\
\quad OCR-Teacher & 0.700 & 0.969 & 0.850 & 5.405 & 1.057 & 3.298 & 8.982 \\
\quad PickScore-Teacher & 0.781 & 0.688 & 0.936 & 6.245 & 1.300 & 3.535 & 11.29 \\
\midrule
\multicolumn{8}{@{}l}{\emph{Per-task Medium teachers}}\\
\quad GenEval-Teacher & 0.917 & 0.558 & 0.845 & 5.336 & 1.074 & 3.332 & 9.024 \\
\quad OCR-Teacher & 0.599 & 0.948 & 0.824 & 5.129 & 0.749 & 3.071 & 6.345 \\
\quad PickScore-Teacher & 0.696 & 0.605 & 0.924 & 6.188 & 1.283 & 3.460 & 11.55 \\
\midrule
DiffusionOPD$_{\mathrm{M\rightarrow M}}$ & 0.908 & 0.944 & 0.917 & 6.067 & 1.275 & 3.461 & 11.07 \\
PDM$_{\mathrm{M\rightarrow M}}$ & 0.902 & 0.940 & 0.913 & 6.124 & 1.261 & 3.455 & 10.97 \\
DiffusionOPD$_{\mathrm{L\rightarrow M}}$ & 0.943 & 0.962 & 0.901 & 6.222 & 1.009 & 3.365 & 8.844 \\
PDM$_{\mathrm{L\rightarrow M}}$ & 0.944 & 0.955 & 0.906 & 6.211 & 1.065 & 3.396 & 9.360 \\
split-KL$_{\mathrm{L\rightarrow M}}$ & 0.939 & 0.959 & 0.912 & 6.211 & 1.103 & 3.417 & 10.08 \\
GFD$_{\mathrm{L\rightarrow M}}$(Ours) & \textbf{0.963} & \textbf{0.973} & \textbf{0.930} & \textbf{6.262} & \textbf{1.302} & \textbf{3.542} & \textbf{11.27} \\
\bottomrule
\end{tabular}}
\end{minipage}\hfill
\begin{minipage}[t]{0.36\textwidth}
\centering
\begin{minipage}[t][10mm][t]{\linewidth}
\caption{SD3.5 single-objective OPD results.}
\label{tab:single}
\end{minipage}\\

\resizebox{\linewidth}{!}{%
\begin{tabular}{@{}l ccc c@{}}
\toprule
& \multicolumn{3}{c}{Task rewards} & \begin{tabular}[c]{@{}c@{}}Image\\quality\end{tabular} \\
\cmidrule(lr){2-4}\cmidrule(lr){5-5}
Model & GenEval & OCR & \begin{tabular}[c]{@{}c@{}}Pick\\Score\end{tabular} & HPSv3 \\
\midrule
SD3.5-Medium (base) & 0.675 & 0.554 & 0.843 & 9.020 \\
\midrule
Teachers (SD3.5-L) & 0.949 & 0.969 & 0.936 & -- \\
\midrule
\multicolumn{5}{@{}l}{\emph{DiffusionOPD$_{\mathrm{L\rightarrow M}}$}}\\
\quad GenEval & 0.948 & -- & -- & 6.844 \\
\quad OCR & -- & 0.967 & -- & 6.097 \\
\quad PickScore & -- & -- & 0.904 & 8.377 \\
\midrule
\multicolumn{5}{@{}l}{\emph{GFD $_{\mathrm{L\rightarrow M}}$(Ours)}}\\
\quad GenEval & \textbf{0.963} & -- & -- & 9.128 \\
\quad OCR & -- & \textbf{0.973} & -- & 8.755 \\
\quad PickScore & -- & -- & \textbf{0.932} & \textbf{11.24} \\
\bottomrule
\end{tabular}}

\vspace{0em}
\caption{FLUX.2 merged OPD results.}
\label{tab:flux}
\resizebox{\linewidth}{!}{%
\begin{tabular}{@{}l ccc c@{}}
\toprule
& \multicolumn{3}{c}{Task rewards} & \begin{tabular}[c]{@{}c@{}}Image\\quality\end{tabular} \\
\cmidrule(lr){2-4}\cmidrule(lr){5-5}
Model & LongText & GenEval & \begin{tabular}[c]{@{}c@{}}Pick\\Score\end{tabular} & HPSv3 \\
\midrule
FLUX.2-4B (base) & 0.624 & 0.769 & 0.823 & 9.214 \\
\midrule
Teacher (FLUX.2-9B) & 0.891 & 0.820 & 0.839 & 9.629 \\
PDM$_{\mathrm{L\rightarrow M}}$ & 0.635 & 0.815 & 0.829 & 8.904 \\
DiffusionOPD$_{\mathrm{L\rightarrow M}}$ & 0.635 & 0.830 & 0.832 & 9.180 \\
GFD$_{\mathrm{L\rightarrow M}}$(Ours) & \textbf{0.692} & \textbf{0.880} & \textbf{0.839} & \textbf{10.20} \\
\midrule
Teacher (FLUX.2-32B) & 0.979 & 0.893 & 0.862 & 11.29 \\
PDM$_{\mathrm{L\rightarrow M}}$ & 0.763 & 0.887 & 0.849 & 10.17 \\
DiffusionOPD$_{\mathrm{L\rightarrow M}}$ & 0.828 & 0.885 & 0.856 & \textbf{10.69} \\
GFD$_{\mathrm{L\rightarrow M}}$(Ours) & \textbf{0.857} & \textbf{0.908} & \textbf{0.859} & 10.47 \\
\bottomrule
\end{tabular}}
\end{minipage}
\vspace{-0.75em}
\end{table}
\textbf{Tasks and evaluation.} We follow the three-reward suite of Flow-GRPO~\citep{flowgrpo} and reuse its official prompt splits: GenEval~\citep{geneval} for compositional alignment, OCR accuracy for visual text rendering, and PickScore~\citep{pickscore} optimized on Pick-a-Pic prompts and evaluated on DrawBench~\citep{drawbench}. Task rewards are the primary metrics; following Flow-GRPO, we also report Aesthetic, ImageReward~\citep{imagereward}, UnifiedReward, and HPSv3 on DrawBench as image-quality probes. All methods are evaluated at $1024\times1024$ with the same prompts, seeds, and sampler, each at its own deployment guidance scale ($w=4.5$ for CFG-composition students, $w=2$ for GFD-OPD; \S\ref{sec:motivation}).

\textbf{Baselines.} (a)~The pretrained SD3.5-Medium and SD3.5-Large base models~\citep{sd3}; (b)~direct Flow-GRPO on both, which also serve as our teachers; (c)~DiffusionOPD, PDM, and the split-KL variant of \S\ref{sec:motivation}, trained with the same teachers, student initialization, and prompts as ours.

\textbf{Implementation details.} All OPD methods share the same on-policy setup: following Fast Flow-GRPO, rollouts use a 10-step reverse-SDE sampler with a 2-step training window. Training uses full-parameter fine-tuning at $1024\times1024$ with learning rate $5\times10^{-6}$ and batch size 48 on 8 H100 GPUs; each single-task GFD-OPD run finishes in about 10 GPU-hours, versus 15 for DiffusionOPD and 160 for direct Flow-GRPO (Fig.~\ref{fig:eff}). GFD-OPD composes the folded target at $w=2$ with extrapolation $\lambda=1.25$.

\subsection{Main Results}\label{sec:main}
\paragraph{Large-to-small distillation on SD3.5 and FLUX.2.}
Table~\ref{tab:single} reports the single-objective SD3.5 setting (Large$\to$Medium). GFD-OPD outperforms DiffusionOPD on all three rewards, surpasses the teacher on GenEval ($0.963$ vs.\ $0.949$) and OCR ($0.973$ vs.\ $0.969$), and is the only student whose HPSv3 stays at or above the untrained base---DiffusionOPD falls clearly below it, the fidelity loss behind the artifacts of \S\ref{sec:motivation}. In the merged multi-objective setting (Table~\ref{tab:main}), the baselines trade one side for the other: with the medium teacher they preserve image quality but are capped by its rewards; with the large teacher their rewards rise but quality drops (haze, desaturation, soft glyphs; Fig.~\ref{fig:artifacts}). GFD-OPD attains the best task rewards ($0.963$ GenEval, $0.973$ OCR, $0.930$ PickScore) and the best image quality at the same time; per \S\ref{sec:abl}, folding alone roughly matches the teacher and the extrapolated target accounts for the rest. The same picture holds on FLUX.2 (Table~\ref{tab:flux}; 9B and 32B teachers into 4B, OCR replaced by our long-text dataset): GFD-OPD leads on all task rewards, again exceeds the teacher on GenEval, and widens its LongText margin at 32B (qualitative examples in Fig.~\ref{fig:flux-quality}, appendix).

\begin{figure}[htbp]
\centering
\includegraphics[width=1\textwidth]{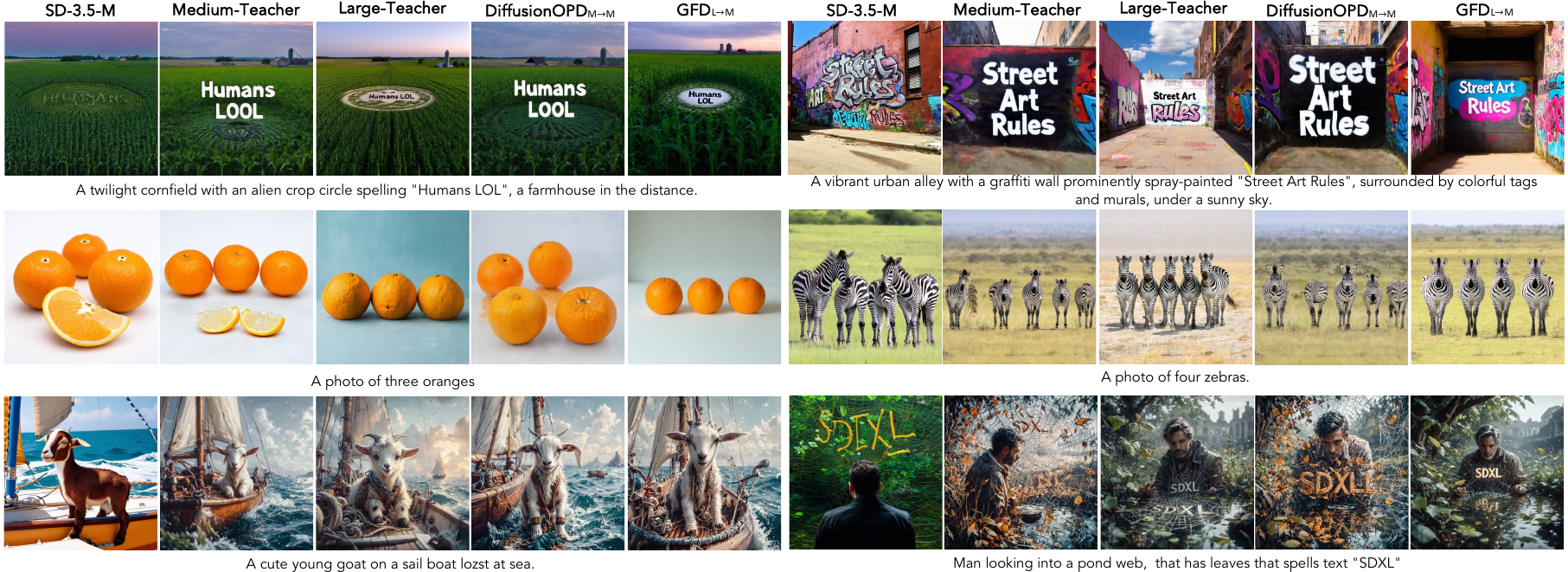}
\vspace{-1em}
\caption{Qualitative comparison at $1024\!\times\!1024$. Columns, left to right: SD3.5-M base; the medium and large teachers; DiffusionOPD$_{\mathrm{M\rightarrow M}}$; and GFD$_{\mathrm{L\rightarrow M}}$.}
\label{fig:quality}
\end{figure}

\paragraph{Qualitative comparison.}
Fig.~\ref{fig:quality} presents a qualitative comparison. Since the large-teacher composition students already suffer from the artifacts shown in Fig.~\ref{fig:artifacts}, we compare with DiffusionOPD distilled from the medium teacher, together with the base model and the two RL teachers. These medium-teacher students are free of visible artifacts, but their prompt adherence and fine-grained rendering are weaker: the crop-circle text in ``Humans LOL'' loses glyph structure, and the orange scene misses part of the required objects. Among all compared methods, GFD-OPD is the closest to the large teacher in both layout and fine detail. Besides, Fig.~\ref{fig:flux-quality} shows qualitative examples for the FLUX.2 experiments of Table~\ref{tab:flux}.

\vspace{-0.5em}
\subsection{Ablations}\label{sec:abl}

\paragraph{Component ablation}
Table~\ref{tab:ablation} adds the two components on top of standard OPD one at a time. Standard OPD improves the base on all three rewards but stays well below the teacher. Guidance folding contributes most of the gain, removing the artifacts of and bringing the student roughly to the teacher level; the extrapolated target adds a further consistent gain and pushes GenEval and OCR past the teacher. This supports the attribution: folding closes the gap to the teacher, and extrapolation provides the margin beyond it.

\begin{table}[h]
\vspace{-1em}
\centering
\caption{Ablation: standard OPD $+$ guidance folding $+$ extrapolated-mean
target (full GFD-OPD), on the SD3.5-Medium student at $1024\times1024$; HPSv3
probes image quality.}
\label{tab:ablation}
\small
\setlength{\tabcolsep}{4pt}
\renewcommand{\arraystretch}{0.9}
\begin{tabular}{lcccc}
\toprule
Method & GenEval & OCR & PickScore & HPSv3 \\
\midrule
SD3.5-Medium (base) & 0.675 & 0.554 & 0.843 & 9.020 \\
Standard OPD & 0.884 & 0.908 & 0.880 & 9.175 \\
\quad $+$ Guidance folding & 0.952 & 0.954 & 0.923 & 11.19 \\
\quad $+$ Extrapolated mean (GFD-OPD) &
\textbf{0.963} & \textbf{0.973} & \textbf{0.930} & \textbf{11.27} \\
\bottomrule
\end{tabular}
\vspace{-1em}
\end{table}

\paragraph{The impact of CFG scale}
In Fig.~\ref{fig:w_sweep}, we evaluate different models under varying CFG scales $w$, measuring the discrepancy between their final sampling results and those of their respective teacher models. DiffusionOPD$_\mathrm{M\rightarrow M}$ approaches the origin at $w=4.5$, indicating a close match to its teacher. In contrast, for DiffusionOPD, PDM, and Split-KL under the $\mathrm{L\!\rightarrow\!M}$ setting, no choice of $w$ brings the results close to the origin. Their outputs remain substantially different from those of the teacher. GFD achieves the closest match to the teacher at $w=2$.

\begin{figure}[h]
\centering
\vspace{-1em}
\includegraphics[width=0.85\linewidth]{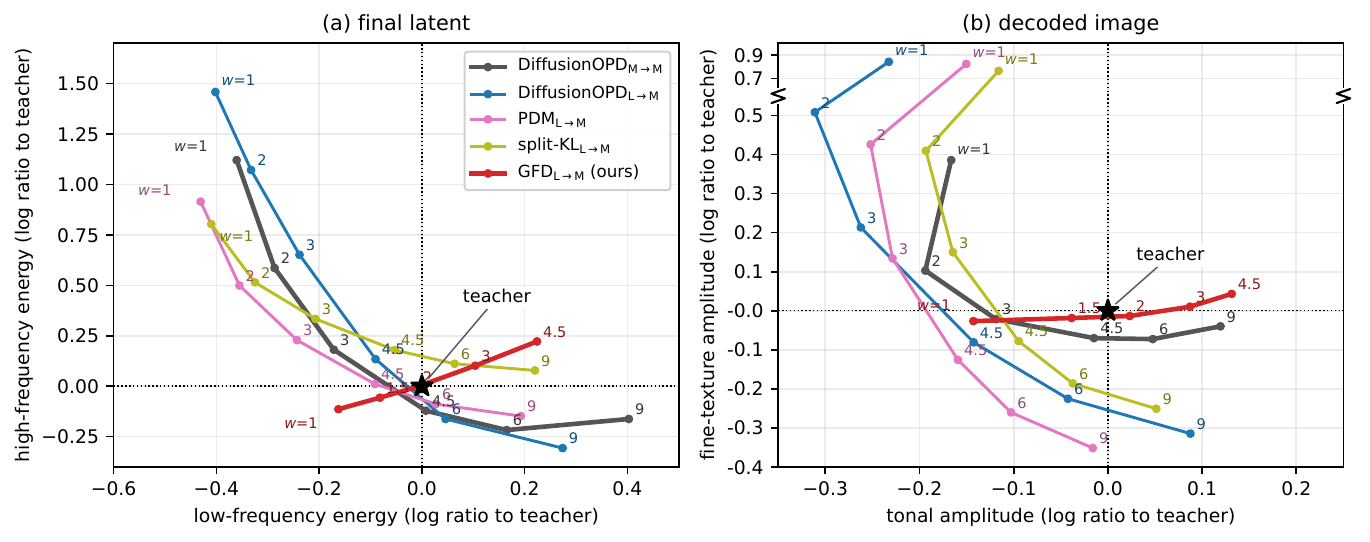}
\vspace{-1em}
\caption{Statistics of the student's final output as the guidance scale $w$ varies. Each axis is the log ratio of a statistic of the student's samples to the teacher's, so the teacher sits at the origin ($\star$).
\textbf{(a)} Final latent: low-frequency energy (global tone) vs.\ high-frequency energy (fine texture). \textbf{(b)} Decoded image: tonal amplitude (RMS contrast $\times$ saturation) vs.\ fine-texture amplitude (3px high-pass MAE). The definition of all metrics are as detailed in Appendix~\ref{subsec:fig2-metrics}.}
\label{fig:w_sweep}
\end{figure}

\vspace{-1.5em}
\section{Related Work}\label{sec:related}
\textbf{On-Policy Distillation.}
OPD~\citep{gkd,minillm,opd} post-trains language models by aligning the student with the teacher on student-generated trajectories, and is used both to compress large teachers into small students and to merge domain specialists into a single model. \citet{exopd} further interpret OPD as dense KL-constrained RL, we adopt this view in \S\ref{sec:gopd}.

\textbf{On-Policy Distillation for Diffusion Models.}
A recent line of work ports OPD to diffusion and flow-matching models~\citep{polyopd,flowopd,stepopd}. DiffusionOPD~\citep{diffusionopd} derives the closed-form per-step objective and merges task-specific teachers into a unified student; PDM~\citep{pdmopd} supervises the conditional and unconditional CFG branches separately. These methods target same-family merging, while the large-to-small regime we study remains largely unexplored.

\section{Conclusion}
We studied on-policy distillation of diffusion models in the large-to-small regime and showed that the
existing recipe does not transfer. To diagnose this we propose Fixed-State KL, which fairly makes methods comparable. 
The analysis attributes the failure to the CFG composition: a small student cannot fit
the teacher's guided field exactly, the composition amplifies the error of its guidance increment by the
guidance scale. GFD-OPD directly absorb the teacher’s guided policy into the student’s conditional branch. 
It reduces the student–teacher gap while avoiding the error amplification of the CFG composition.
We hope this study can encourage further investigation into how the distributional discrepancy between teacher and student affect OPD of diffusion models.

\bibliography{iclr2027_conference}
\bibliographystyle{iclr2027_conference}

\newpage
\appendix
\section{Derivation}

\subsection{Reward extrapolation and the extrapolated target}\label{app:exopd}

\paragraph{Reward extrapolation in language models.}
\citet{exopd} generalize the OPD objective (Eq.~\ref{eq:opd-llm}) by introducing a frozen reference policy $\pi_{\mathrm{ref}}$ and a reward weight $\lambda$:
\begin{equation}
J(\bm\theta)=
\mathbb{E}_{\bm y\sim\pi_{\bm\theta}(\cdot|\bm x)}
\Big[\lambda\log\frac{\pi^\star(\bm y|\bm x)}{\pi_{\mathrm{ref}}(\bm y|\bm x)}\Big]
-\mathrm{KL}\big(\pi_{\bm\theta}(\cdot|\bm x)\,\|\,\pi_{\mathrm{ref}}(\cdot|\bm x)\big).
\label{eq:gopd-llm}
\end{equation}
Standard OPD is the case $\lambda=1$ (the reference cancels). Define
$\pi_{\mathrm{tgt}}(\bm y|\bm x)=(\pi^\star)^{\lambda}(\pi_{\mathrm{ref}})^{1-\lambda}/Z(\bm x)$ with
$Z(\bm x)=\sum_{\bm y}(\pi^\star)^{\lambda}(\pi_{\mathrm{ref}})^{1-\lambda}$. Expanding the KL term,
\begin{align}
J(\bm\theta)
&=\mathbb{E}_{\pi_{\bm\theta}}\big[\lambda\log\pi^\star-\lambda\log\pi_{\mathrm{ref}}-\log\pi_{\bm\theta}+\log\pi_{\mathrm{ref}}\big]\nonumber\\
&=\mathbb{E}_{\pi_{\bm\theta}}\big[\lambda\log\pi^\star+(1-\lambda)\log\pi_{\mathrm{ref}}-\log\pi_{\bm\theta}\big]
=-\,\mathrm{KL}\big(\pi_{\bm\theta}\,\|\,\pi_{\mathrm{tgt}}\big)+\log Z(\bm x).
\label{eq:gopd-identity}
\end{align}
Since $Z(\bm x)$ does not depend on $\bm\theta$, maximizing $J$ is the same as minimizing the reverse KL to $\pi_{\mathrm{tgt}}$; that is, ExOPD is standard OPD with the teacher replaced by $\pi_{\mathrm{tgt}}$, where
\begin{equation}
\log\pi_{\mathrm{tgt}}
=\log\pi^\star+(\lambda-1)\big(\log\pi^\star-\log\pi_{\mathrm{ref}}\big)+\mathrm{const}.
\label{eq:exopd-llm}
\end{equation}
For $0<\lambda<1$ the target interpolates between reference and teacher; for $\lambda>1$ it extrapolates along the reference$\to$teacher direction.

\paragraph{Diffusion.}
At denoising step $j$, the teacher and reference kernels are $p_T=\mathcal N(\bm\mu_T,\sigma_j^2\bm I)$ and $p_R=\mathcal N(\bm\mu_R,\sigma_j^2\bm I)$ with the same sampler-fixed variance (\S\ref{sec:prelim}). Applying the replacement above to the kernels gives $p_{\mathrm{tgt}}\propto p_T^{\lambda}p_R^{1-\lambda}$.

\begin{lemma}\label{lem:geo}
Let $p_1=\mathcal N(\bm\mu_1,\sigma^2\bm I)$ and $p_2=\mathcal N(\bm\mu_2,\sigma^2\bm I)$. For every $\lambda\in\mathbb R$, the normalized product $p\propto p_1^{\lambda}p_2^{1-\lambda}$ equals $\mathcal N\big(\lambda\bm\mu_1+(1-\lambda)\bm\mu_2,\ \sigma^2\bm I\big)$.
\end{lemma}
\begin{proof}
$\log p(\bm x)\overset{c}{=}-\tfrac{\lambda}{2\sigma^2}\|\bm x-\bm\mu_1\|_2^2-\tfrac{1-\lambda}{2\sigma^2}\|\bm x-\bm\mu_2\|_2^2$.
The coefficient of $\|\bm x\|_2^2$ is $-\tfrac{\lambda+(1-\lambda)}{2\sigma^2}=-\tfrac{1}{2\sigma^2}$, independent of $\lambda$, so the covariance is $\sigma^2\bm I$; the linear term is $\tfrac{1}{\sigma^2}\langle\bm x,\lambda\bm\mu_1+(1-\lambda)\bm\mu_2\rangle$, and completing the square gives the mean.
\end{proof}

By Lemma~\ref{lem:geo}, $p_{\mathrm{tgt}}=\mathcal N\big(\lambda\bm\mu_T+(1-\lambda)\bm\mu_R,\ \sigma_j^2\bm I\big)$: extrapolation changes the mean only, and $p_{\mathrm{tgt}}$ is a valid transition kernel for every $\lambda$, including $\lambda>1$. Because the transition mean is affine in the velocity with sampler-fixed coefficients, $\bm\mu=\alpha_j\bm x_{t_j}+\beta_j v$, the target mean corresponds to $v_{\mathrm{tgt}}=\lambda v_T+(1-\lambda)v_R$ (Eq.~\ref{eq:v-tgt}). The reverse KL between the student kernel and $p_{\mathrm{tgt}}$ is a KL between Gaussians with equal covariance, $\|\bm\mu_S-\bm\mu_{\mathrm{tgt}}\|_2^2/(2\sigma_j^2)$, which in velocity space is the weighted matching loss of Eq.~\ref{eq:opd-diff} with the target velocity $v_{\mathrm{tgt}}$ in place of $v^g_T$.

\section{Detailed Analysis}
\vspace{-0.5em}
\subsection{Analysis of error cancellation}~\label{subsec:errpr_decompose}
\vspace{-2.5em}
\paragraph{The two error components.}
Given the guided output
\[
v_g=v_u+w(v_c-v_u),
\]
the guided prediction error can be written as
\[
\delta^g=\delta^u+wd,
\]
where
\[
d=\delta_c-\delta_u.
\]
Alternatively, $d$ can be expressed as
\[
d=\Delta_S-\Delta_T,
\]
where
\[
\Delta_S=v_S^c-v_S^u,\qquad
\Delta_T=v_T^c-v_T^u.
\]
These two error components have distinct meanings. Specifically, $\delta^u$ represents the prompt-independent denoising prediction error, while $d$ captures the error in the prompt-induced increment.

\begin{table}[h]
\centering\small
\caption{Decomposition $\bm \delta^g=\bm\delta^{\mathrm u}+w\bm d$ of the guided error at $w\!=\!4.5$. 
Cross term$=2\langle\bm\delta^{\mathrm u},4.5\bm d\rangle/\|\bm r\|^2$.}
\label{tab:mot-decomp}
\begin{tabular}{l ccccc}
\toprule
Student & $\|\bm\delta^{\mathrm u}\|$ & $\|\bm d\|$ & $\|\bm r\|$ & cross term \\
\midrule
DiffusionOPD$_{\mathrm{M\rightarrow M}}$ (reference) & 0.092 & 0.032 & \textbf{0.106} & $\bm{-1.56}$ \\
DiffusionOPD$_{\mathrm{L\rightarrow M}}$ & 0.090 & 0.043 & 0.177 & $-0.48$ \\
PDM$_{\mathrm{L\rightarrow M}}$ & \textbf{0.053} & 0.041 & 0.178 & $-0.18$ \\
split-KL$_{\mathrm{L\rightarrow M}}$ & \textbf{0.051} & 0.043 & 0.185 & $-0.18$ \\
\bottomrule
\end{tabular}
\label{tab:error_decompose}
\end{table}

\paragraph{Cancellation in $\mathrm{M\rightarrow M}$ and its loss in $\mathrm{L\rightarrow M}$.}
The guided objective only constrains the sum of the two components, $\delta^u+wd$. Therefore, any pair of errors satisfying $d=-\delta^u/w$ produces the same guided prediction error and lies in the null space of the objective. Consequently, minimizing this objective does not necessarily force both error components to vanish simultaneously. Instead, it allows the student to learn a negative correlation between $d$ and $\delta^u$, enabling one component to compensate for the other.

$\mathrm{DiffusionOPD}_\mathrm{M\rightarrow M}$ nearly reaches such a configuration, where the two error components $\delta^u$ and $d$ substantially cancel each other. For $\mathrm{DiffusionOPD}_\mathrm{L\rightarrow M}$, besides a similar shrinkage effect, the error in $d$ contains additional components that are orthogonal to $\delta^u$, resulting in weaker error cancellation. 

PDM and Split-KL explicitly supervise the two branches separately. Although this reduces the magnitude of the unconditional error component $\delta^u$, it also makes the two error components nearly orthogonal to each other. As a result, the cancellation effect is largely removed, and the final guided error remains substantial.

\subsection{Shrinkage of the trained field along the teacher's guided field}\label{sec:app-proj}

Tab.~\ref{tab:app-proj} quantifies the residual shrinkage of the guidance-folded student and how the
inference-time guidance increment compensates it. At every step we project the student's deployed field
onto the teacher's guided field, $g=\langle\bm v_S(w),\bm v_T^{\cfg}\rangle/\|\bm v_T^{\cfg}\|^2$, so
that $g\!=\!1$ means the student carries the teacher's full amplitude along the teacher's direction.
Results in the table is averaged over 500 prompts and over the steps of each window.

\begin{table}[h]
\centering\small
\caption{Projection $g$ of the student's deployed field onto the teacher's guided field.}
\label{tab:app-proj}
\begin{tabular}{l l c c c c c}
\toprule
Student & Probe & 0--9 & 10--19 & 20--29 & 30--37 & 0--37 (total)\\
\midrule
GFD ($w\!=\!1$) & $\mathrm{P_{tea}}$  & 0.963 & 0.984 & 0.991 & 0.983 & 0.980\\
                & $\mathrm{P_{real}}$  & 0.943 & 0.971 & 0.987 & 0.984 & 0.968\\
 
\midrule
GFD ($w\!=\!2$) & $\mathrm{P_{tea}}$   & 1.030 & 0.991 & 0.993 & 0.984 & \bf{1.000}\\
                & $\mathrm{P_{real}}$  & 1.023 & 0.978 & 0.989 & 0.985 & \bf{0.997}\\
 
\bottomrule
\end{tabular}
\end{table}

At $w\!=\!1$ the trained field of GFD is short along $\bm v_T^{\cfg}$ in every window, most so at the high-noise steps $0$--$9$.
The student has thus learned the direction of the
teacher's guidance but shrinks its magnitude toward its own prior. Raising the guidance scale to
$w\!=\!2$ compensates the shrinkage, and the
average over steps $0$--$37$ moves from $0.97$--$0.98$ to $1.00$ on both probe sets.
Besides, in the low-noise regime, the unconditional branch co-evolves with the conditional branch and exhibits similar shifts, leading to only a marginal increase of the CFG increment and avoiding excessive guidance.

\subsection{Metrics of Fig.~\ref{fig:w_sweep}}\label{subsec:fig2-metrics}

All four quantities compare a student sample with the teacher sample generated from the same prompt
and the same initial noise, over $N\!=\!500$ prompts; each axis of Fig.~\ref{fig:w_sweep} is a
log ratio to the teacher, so the teacher sits at the origin.

\paragraph{Latent frequency energies (panel a).}
Let $z\in\mathbb R^{64\times64\times64}$ be the final latent in the model's packed token layout ($64$
channels on a $64\times64$ grid) and $\hat z_c=\mathrm{DCT}_2(z_c)$ the orthonormal 2-D DCT-II of channel
$c$. With $u,v\in\{0,\tfrac1{63},\dots,1\}$ the normalized frequency indices and
$\rho(u,v)=\sqrt{u^2+v^2}/\sqrt2\in[0,1]$ the radial frequency, the energy of a band $B$ is
\begin{equation}
E_B(z)=\frac{1}{64^3}\sum_{c}\sum_{(u,v)\in B}\hat z_c(u,v)^2,\qquad
B_{\mathrm{low}}=\{\rho<0.25\},\quad B_{\mathrm{high}}=\{\rho\ge0.75\}.
\end{equation}
The two coordinates of panel (a) are
$\log\big(\tfrac1N\sum_n E_{B}(z^S_n)/E_{B}(z^T_n)\big)$ for $B=B_{\mathrm{low}}$ (horizontal) and
$B=B_{\mathrm{high}}$ (vertical), where $z^S_n$ and $z^T_n$ are the student's and the teacher's final
latents for prompt $n$.

\paragraph{Image tonal and fine-texture amplitude (panel b).}
The decoded $1024^2$ image $I$ is area-downsampled to $512^2$ and converted to luminance
$Y=0.2126R+0.7152G+0.0722B$. Tonal amplitude combines the RMS contrast and the mean saturation,
\begin{equation}
c(I)=\frac{\mathrm{std}(Y)}{\mathrm{mean}(Y)},\qquad
s(I)=\operatorname*{mean}_{p}\frac{\max_{k}I_k(p)-\min_{k}I_k(p)}{\max_{k}I_k(p)},\qquad
a_{\mathrm{tone}}(I)=\sqrt{c(I)\,s(I)},
\end{equation}
with $k$ ranging over the RGB channels and $p$ over pixels. Fine-texture amplitude is the mean
magnitude of the $3\times3$ high-pass residual of the luminance,
\begin{equation}
a_{\mathrm{tex}}(I)=\operatorname*{mean}_{p\in\Omega}\big|Y(p)-\mathrm{box}_{3\times3}Y(p)\big|,
\end{equation}
where $\mathrm{box}_{3\times3}$ is the $3\times3$ mean filter with reflect padding and $\Omega$ excludes a
$16$-pixel border. The two coordinates of panel (b) are the mean log ratios
$\tfrac1N\sum_n\log\big(a_{\mathrm{tone}}(I^S_n)/a_{\mathrm{tone}}(I^T_n)\big)$ (horizontal) and
$\tfrac1N\sum_n\log\big(a_{\mathrm{tex}}(I^S_n)/a_{\mathrm{tex}}(I^T_n)\big)$ (vertical).

\section{More Ablations}\label{app:more-ablations}

\subsection{GFD with unconditonal branch supervision}
\paragraph{Supervising the unconditional branch (GFD$+$uncond).}
GFD only constrains the conditional branch and leaves the unconditional branch unconstrained. A natural
variant, GFD$+$uncond, keeps the folded target for the conditional branch and additionally
matches the student's unconditional branch to the teacher's:
\begin{equation}
    \mathcal{L}(\bm\theta)
    =\mathbb{E}_{x_{0:N}\sim p_{S,\bm\theta}}
    \left[\sum_{j=0}^{N-1}
    \omega_j\Big(\left\| v_S^c - v_T^g \right\|_2^2
    +
    \left\| v_S^u - v_T^u \right\|_2^2\Big)
    \right].
\end{equation}

\paragraph{Same starting point.}
At $w\!=\!1$ the two students deploy fields trained toward the same target and are indistinguishable:
their KL divergence values are very similar, and their final sampling results sit at the same point
in Fig.~\ref{fig:w_sweep_uncond}.

\paragraph{Wrong increment direction.}
What changes is the content of the CFG increment $v^c_S-v^u_S$. In GFD the unconditional branch co-evolves with the conditional branch and exhibits similar shifts, resulting in the CFG only containing a small increment and avoiding excessive guidance. Raising $w$ mainly adds a native-scale texture correction. However, GFD$+$uncond matches the unconditional branch to $v^u_T$, removing the co-evolving. Its CFG increment is consequently doubled in magnitude. As demonstrated in Tab.~\ref{tab:gfduncond} and Fig.~\ref{fig:w_sweep_uncond}, this causes a larger KL divergence and drives the final sampled results further away from the teacher distribution.

\begin{table}[h]
\centering
\caption{KL divergence to the teacher for GFD and GFD$+$uncond, in the format of Tab.~\ref{tab:kl_result}.}
\label{tab:kl_uncond}

\setlength{\tabcolsep}{4pt}
\renewcommand{\arraystretch}{0.9}

\begin{tabular}{lcc}
\toprule
Student & KL ($\mathrm{P_{tea}}$) & KL ($\mathrm{P_{real}}$)\\
\midrule
GFD$_{\mathrm{L\rightarrow M}}$ ($w\!=\!1$) & 0.050 & 0.030\\
GFD$_{\mathrm{L\rightarrow M}}$ ($w\!=\!2$) & 0.053 & 0.032\\
\midrule
GFD+uncond$_{\mathrm{L\rightarrow M}}$ ($w\!=\!1$) & 0.048 & 0.029\\
GFD+uncond$_{\mathrm{L\rightarrow M}}$ ($w\!=\!2$) & 0.060 & 0.040\\
\bottomrule
\end{tabular}
\label{tab:gfduncond}
\end{table}

\begin{figure}[h]
\centering
\includegraphics[width=\linewidth]{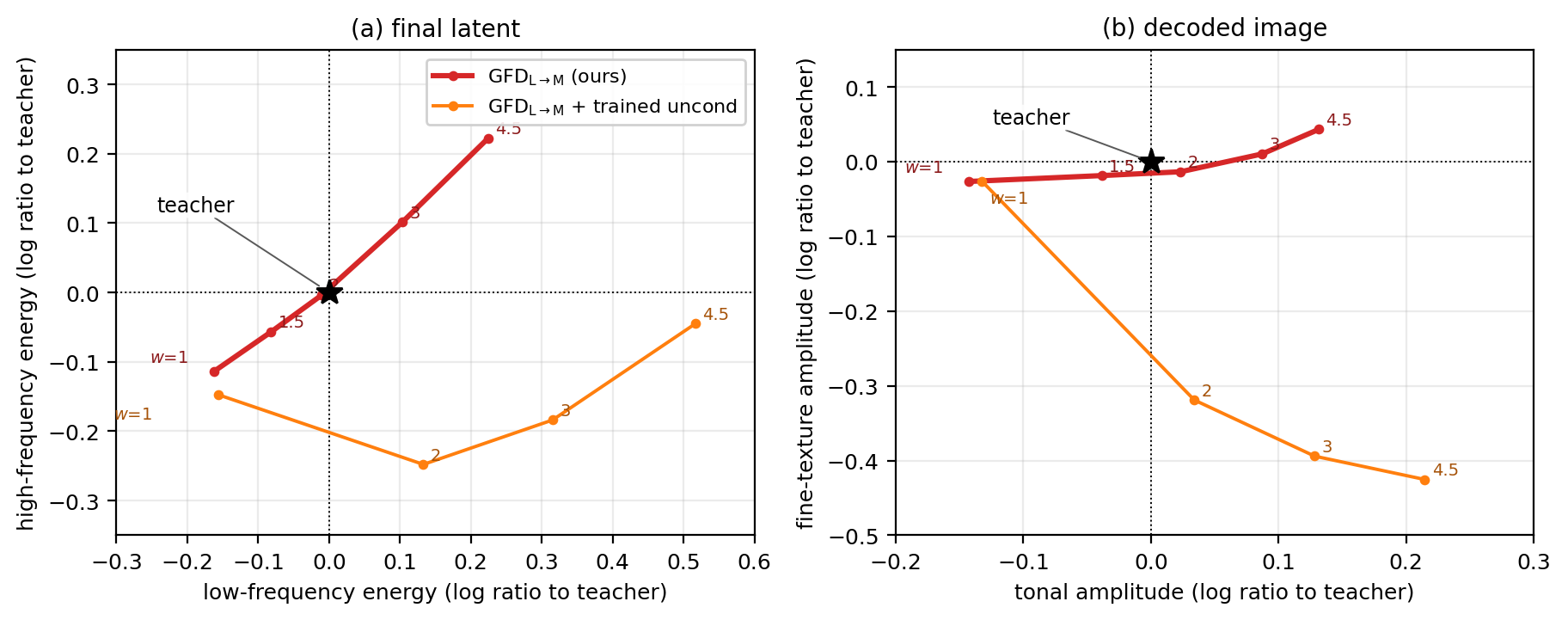}
\caption{Statistics of the final output of GFD and GFD$+$uncond as the guidance scale $w$ varies, in the
coordinates of Fig.~\ref{fig:w_sweep} (log ratio to the teacher, teacher at the origin). }
\label{fig:w_sweep_uncond}
\end{figure}

\section{Additional qualitative results on FLUX.2}\label{app:flux-quality}

\begin{figure}[htbp]
\centering
\includegraphics[width=\textwidth]{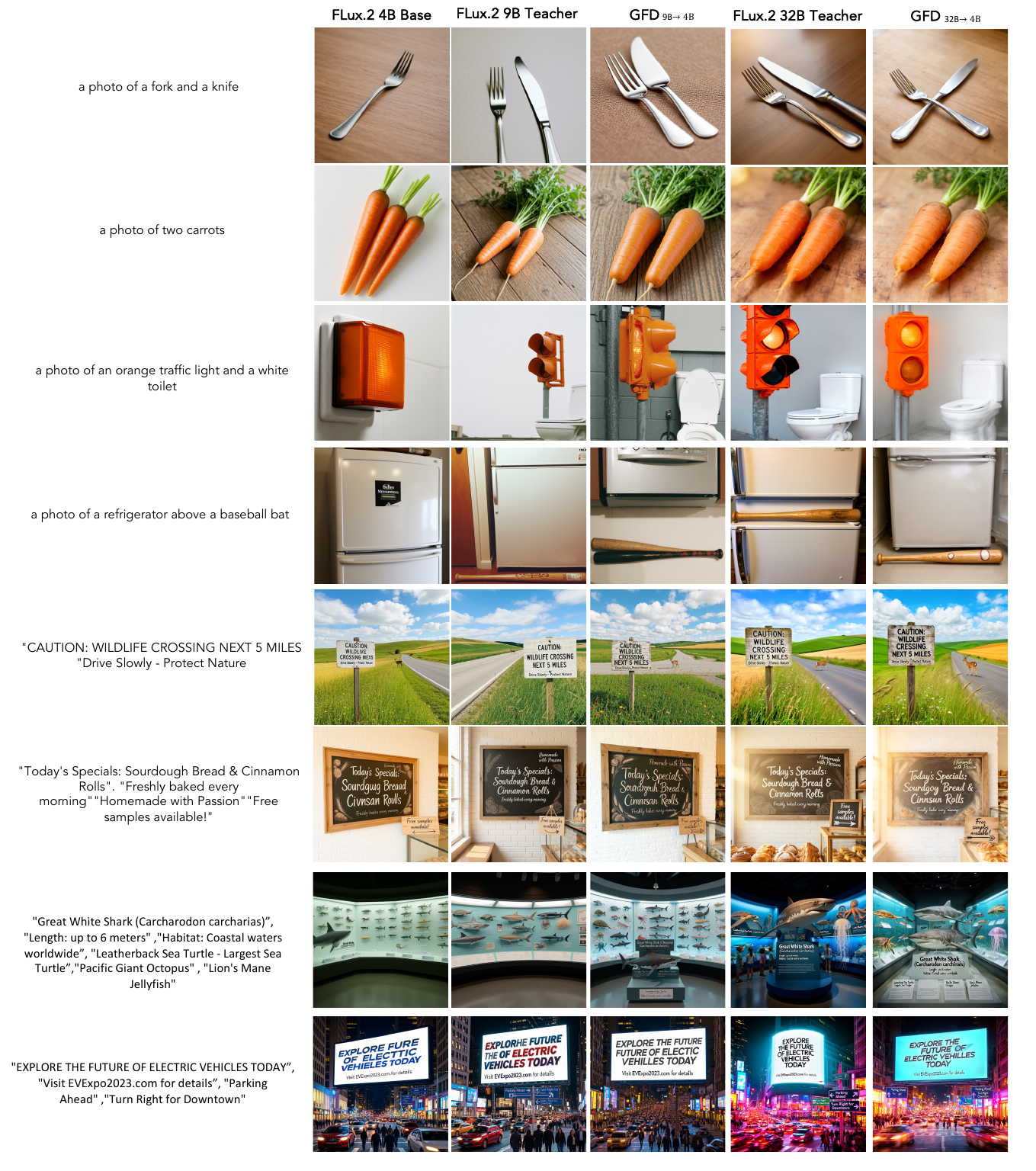}
\caption{Qualitative comparison on FLUX.2 at $1024\!\times\!1024$, complementing Table~\ref{tab:flux}. Columns, left to right: the FLUX.2-4B base model; the FLUX.2-9B teacher and the GFD-OPD student distilled from it into 4B; the FLUX.2-32B teacher and the GFD-OPD student distilled from it into 4B. The top four rows are GenEval prompts; the bottom four are LongText prompts probing long-form visual text rendering.}
\label{fig:flux-quality}
\end{figure}

\end{document}